\documentclass[10pt, conference, letterpaper]{IEEEtran}
\IEEEoverridecommandlockouts
\usepackage{cite}
\usepackage{amsmath,amssymb,amsfonts}
\usepackage{graphicx}
\usepackage{textcomp}
\usepackage{xcolor}
\usepackage{color}
\usepackage{url}
\usepackage{bbm}
\usepackage{algorithm}
\usepackage[noend]{algorithmic}
\usepackage{bm}
 \usepackage{enumitem}
 \usepackage{amsthm}
 \usepackage{amsmath}
 \usepackage{longtable}
 \usepackage[small]{caption}
 \usepackage{tabularx}
 \usepackage{multirow}
\usepackage{eqparbox}
\usepackage{mathtools}
 \usepackage{subfigure}
\usepackage[utf8]{inputenc}

\renewcommand{\algorithmicrequire}{\textbf{Input:}}

\usepackage{diagbox} 
\usepackage{booktabs} 

\usepackage{xspace}
\newcommand{\ouralg}{ORRIC\xspace}
\usepackage{tikz}
\newcommand{\WoB}[1]{{\small \tikz[baseline=(char.base)]{\node[shape=circle,fill=black,draw,inner sep=0.5pt] (char) {\color{white}#1};}}}

\usepackage{thmtools}

\usepackage{hyperref}

\hypersetup{
    colorlinks=true,
    linkcolor=black,
    urlcolor=blue,
}

\declaretheoremstyle[%
  spaceabove=3pt,%
  spacebelow=3pt,%
  headfont=\normalfont\bfseries,%
  bodyfont=\normalfont\itshape,%
  postheadspace=0.5em,%
]{theoremstyle} 

\declaretheorem[name={Definition},style=theoremstyle]{definition}
\declaretheorem[name={Assumption},style=theoremstyle]{assumption}
\declaretheorem[name={Theorem},style=theoremstyle]{theorem}
\declaretheorem[name={Corollary},style=theoremstyle]{corollary}
\declaretheorem[name={Lemma},style=theoremstyle]{lemma}

\declaretheorem[name={Property},style=theoremstyle]{property}

\declaretheoremstyle[%
  spaceabove=0pt,%
  spacebelow=3pt,%
  headfont=\normalfont\itshape,%
  postheadspace=1em,%
  qed=\qedsymbol%
]{proofstyle}

\def\BibTeX{{\rm B\kern-.05em{\sc i\kern-.025em b}\kern-.08em
    T\kern-.1667em\lower.7ex\hbox{E}\kern-.125emX}}  
\begin{document} 
\title{Online Resource Allocation for Edge Intelligence with Colocated Model Retraining and Inference} 

\author{
\IEEEauthorblockN{Huaiguang Cai, Zhi Zhou,  Qianyi Huang}
\IEEEauthorblockA{School of Computer Science and Engineering, Sun Yat-Sen University, China}
\IEEEauthorblockA{Email: caihg3@mail2.sysu.edu.cn, zhouzhi9@mail.sysu.edu.cn, huangqy89@mail.sysu.edu.cn}
}   
\maketitle 
\begin{abstract}
With edge intelligence, AI models are increasingly pushed to the edge to serve ubiquitous users. However, due to the drift of model, data, and task, AI model deployed at the edge suffers from degraded accuracy in the inference serving phase. Model retraining handles such drifts by periodically retraining the model with newly arrived data. When colocating model retraining and model inference serving for the same model on resource-limited edge servers, a fundamental challenge arises in balancing the resource allocation for model retraining and inference, aiming to maximize long-term inference accuracy. This problem is particularly difficult due to the underlying mathematical formulation being time-coupled, non-convex, and NP-hard. To address these challenges, we introduce a lightweight and explainable online approximation algorithm, named \ouralg, designed to optimize resource allocation for adaptively balancing the accuracy of model training and inference. The competitive ratio of \ouralg outperforms that of the traditional Inference-Only paradigm, especially when data drift persists for a sufficiently lengthy time. This highlights the advantages and applicable scenarios of colocating model retraining and inference. Notably, \ouralg can be translated into several heuristic algorithms for different resource environments. Experiments conducted in real scenarios validate the effectiveness of \ouralg.

\end{abstract}

\section{Introduction}

Edge intelligence, the marriage of AI and edge computing, promises to provide ubiquitous users with low-latency, energy-efficient, and privacy-protecting machine learning services by processing data in proximity\cite{zhizhou_ei_20}. However, various types of drift reduce the accuracy of machine learning models in practice, and even worse in edge scenarios when computing resources are limited. Specifically, we classify these drifts into three types: 1)  \emph{model drift}: the distribution of model parameters is changed after deployment (e.g., model compression).  2) \emph{data drift}: the distribution of features or labels shift over time (e.g., domain adaptation, test-time adaptation \cite{etta_niu_22}).  3) \emph{task drift}: the model may be applied to perform unseen tasks (e.g., fine tuning\cite{CNN_taj_16}, embodied AI\cite{embodied_AI}).


Numerous methods have been proposed to alleviate these drifts, including retraining the deployed model \cite{CNN_taj_16, etta_niu_22, ttasurvey_liang_23, concept_jie_19} or modifying inference results based on certain data distribution assumptions \cite{labelshift_wu_21, timecary_rasool_22}. However, these methods, mainly proposed by researchers from the machine learning community, tend to emphasize accuracy while overlooking resource consumption. Fortunately, in the field of edge computing, significant research such as Ekya\cite{ekya_romil_22}, RECL\cite{RECL_mehrdad_23} and Shoggoth\cite{wang2023shoggoth}, has been proposed to handle drifts by navigating the trade-off between the tasks of model retraining and model inference under the constraints of limited edge resources. Here, we define the scheme of \textbf{retraining the model and performing inference simultaneously on new data} as the \emph{model retraining and inference co-location} paradigm.

Nevertheless, the absence of formal modeling and the reliance on heuristic algorithms in these previous works limit our understanding of the model retraining and inference co-location paradigm. Modeling this paradigm not only provides insights into its advantages and application scenarios but also aids in designing more rational and explainable algorithms that may enjoy better performance and theoretical guarantees.



Intuitively, due to limited edge resources, the task of retraining the model on new data and the task of performing inference on new data form a competitive relationship. If there are more retraining resources currently assigned, the current inference accuracy is low and the future accuracy is high; on the contrary, if the retraining resources are currently assigned less, the current inference accuracy is high but the future accuracy is low. Then a central question arises:

\emph{How can resources be credibly allocated for model retraining and inference co-location to optimize long-term model performance under various drifts?}

To answer this question, our work makes the following contributions:

\begin{enumerate}
\item We provide a natural modeling of the model retraining and inference co-location paradigm and demonstrate a corresponding typical and practical system (Section \ref{sec:pf}).

\item  We design a lightweight and explainable algorithm \ouralg  (Section \ref{sec:ALG}) for the paradigm. The proved competitive ratio of \ouralg is strictly better than that of the traditional Inference-Only paradigm when data drift occurs for a sufficiently lengthy time, implying the advantages and application scenarios of model retraining and inference co-location paradigm (Section \ref{sec:pa}).


\item Our experimental results of \ouralg on CIFAR-10-C validate the effectiveness of model retraining and inference co-location in drift scenarios (Section \ref{sec:exp}). Our code is available at \url{https://github.com/caihuaiguang/ORRIC}.
\end{enumerate}


\section{Background and Related Works}
We motivate our work with prior studies on 1) drifts in machine learning and 2) inference and retraining configuration adapting in edge computing.

\subsection{Drift in Machine Learning}


The basic process of machine learning is to collect a large amount of data for a task and then use the data to train a machine learning model. However, in practice, the model, data, and task may change after the deployment of the model. We categorize the inconsistency between the training phase and inference phase as model drift, data drift, and task drift.
  
\subsubsection{Model drift} \label{sec:modeldrift}


DNN compression\cite{DC_Han_15} is commonly adopted for lower latency and improved energy efficiency\cite{zhizhou_ei_20}. However, the distribution of model parameters is changed \cite{model_dong_22} after compression, usually leading to a decrease in the accuracy of model. We classify this inconsistency between the model training phase and the inference phase as model drift. Even though model performance on training data remains the same after the compression, the compressed model has less generalization power on unseen data\cite{cost_jia_22}, necessitating model retraining\cite{ekya_romil_22}.
 

\subsubsection{Data Drift}
This type of drift represents a shift in the distribution of features or labels. Specifically, let $X$ denote the feature vector and $y$ denote the label. We use $P_t(X,y)$ to represent the joint probability density function of $X,y$ at time $t$. Concept drift \cite{concept_jie_19} occurs when there exists a time $t$ such that $P_t(X,y)\ne P_{t+1}(X,y)$. Label shift\cite{labelshift_wu_21} occurs when there exists a time $t$ such that $P_t(y)\ne P_{t+1}(y)$ but for all $t$, $P_t(X|y)= P_{t+1}(X|y)$. In a broader sense, we classify any shift in $P_t(X)$, $P_t(y)$, $P_t(X|y)$, $P_t(y|X)$, $P_t(X,y)$ (such as concept drift, label shift, domain adaptation, or test-time adaptation\cite{etta_niu_22}) as data drift. In real-world scenarios, the proportion of pedestrians, cars, and bicycles may vary throughout the day\cite{ekya_romil_22}, corresponding to label shift ($P_t(y)$ shift). Another common phenomenon is that, due to variations in angles, weather conditions, lighting, and sensors\cite{ekya_romil_22}, the appearance of the same class of objects ($P_t(X|y)$) may differ from that during training, while the true label of the object ($P_t(y)$) remains unchanged, corresponding to concept drift.

Several methods, such as the unbiased risk estimator \cite{labelshift_wu_21} and the online ensemble algorithm \cite{ ttasurvey_liang_23}, can alleviate the negative effect of data drift on model performance without retraining based on assumptions on data distribution, such as label shift \cite{labelshift_wu_21} or gradual data evolution \cite{timecary_rasool_22}. However, these methods are heavily dependent on the assumed type of drift and may not be universally applicable. As a result, retraining the model \cite{etta_niu_22} remains a mainstream approach \cite{ttasurvey_liang_23}.

   
\subsubsection{Task drift}   
This drift encompasses changes in tasks during both the training and inference phases, including meta-learning\cite{meta_chen_22}, continual learning \cite{etta_niu_22}, transfer learning\cite{ttasurvey_liang_23}, and fine-tuning\cite{CNN_taj_16}. Additionally, embodied AI has gained considerable attention recently, necessitating models to learn through interactions with the real world\cite{embodied_AI}. All these studies expect the model to perform well on new tasks, and retraining the model is nearly the only viable approach.
 


Motivated by the extensive attention to drift problems in the machine learning research community and the prevalent use of model retraining, we seek to formulate the model performance under these drifts.  Moreover, in contrast to prior works that concentrate solely on model accuracy, our approach considers both model accuracy and the computational cost of model retraining. This makes it more suitable for practical deployment in resource-constrained edge environments.



\subsection{Inference and Training Configuration Adaption}
Resources provisioned for edge computing are limited\cite{ekya_romil_22}, motivating research on reducing resource consumption for model retraining or inference on edge while meeting basic accuracy requirements, known as model inference or retraining configuration adaptation.

\emph{Inference Configuration Adaptation} refers to adapting the content of the inference request or the model used, such as the frame rate of the input video, the resolution of input images\cite{coordi_peng_20}, the type of model\cite{adaptor_kong_22}, or the extent of early exiting in model inference\cite{lien_20}. This adaptation, influencing the corresponding output of the model, is typically determined by the available computing, storage, bandwidth resources, and the difficulty level of the input\cite{coordi_peng_20}.
                      

\emph{Training Configuration Adaption} refers to adapting the hyperparameters of the training, such as epochs, training data size\cite{real_khani_21}, or the layers performing back-propagation\cite{etta_niu_22}. This adaptation, influencing the model itself, is typically determined by available computing, storage, bandwidth resources, and the performance of the model being used\cite{ekya_romil_22}.





Although some studies such as Ekya\cite{ekya_romil_22} and RECL\cite{RECL_mehrdad_23} have explored the trade-off between the tasks of model retraining and model inference under the constraints of limited edge resources, they lack the formal modeling of the model retraining and inference co-location paradigm, and the employed algorithms are heuristic. Differing from these works, we aim to gain a deeper understanding of the paradigm and design a more rational and explainable algorithm by formally modeling the paradigm and proposing a theoretically guaranteed algorithm.

 
\noindent\textbf{Remark}: Existing researches on model retraining and inference co-location typically deploy the model on edge \cite{ekya_romil_22} or cloud\cite{RECL_mehrdad_23}. However, we argue that with hardware upgrades \cite{Tiny_Shi_22} and technological advances \cite{Mandheling_Xu_22}, model retraining and inference co-location on devices holds promise for enhanced privacy protection, reduced bandwidth usage and personalized AI models. While some existing frameworks support on-device model training, such as MNN \cite{walle_lv_22}, nntrainer \cite{nntrainer}, TensorFlow Lite \cite{Lite}, PyTorch Mobile \cite{Mobile}, their documentation is not comprehensive, and some lack regular maintenance. We call for further efforts in this direction.


 
                  
\section{System Model and Problem Formulation for Model Retraining and Inference Co-location} \label{sec:pf} 
In this section, we present the system model for model retraining and inference colocation in an edge server, and the problem formulation for the dynamical resource allocation to maximize the long-term inference accuracy.

\subsection{System Overview}\label{subsc:sys}

\begin{figure}[t]
    \centerline{\includegraphics[width=0.48\textwidth]{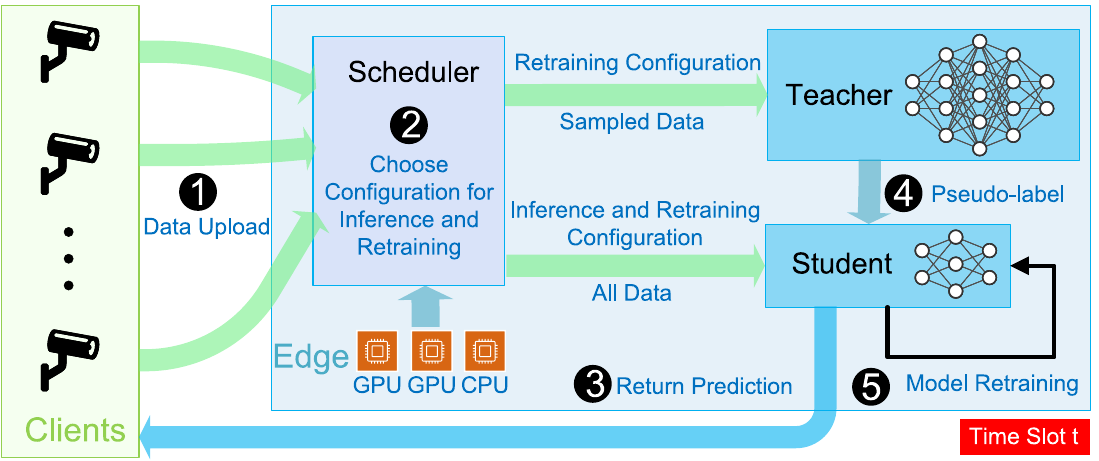}} 
    \caption{Model Retraining and Inference Co-location.} 
    \label{fig: Model}
    \vspace{-0.5cm}
\end{figure}


  
Following the pilot effort of Ekya\cite{ekya_romil_22}, we adopt a general system architecture as illustrated in Fig. \ref{fig: Model} for edge intelligence with colocated model training and inference. In this architecture, an edge server equipped with moderate CPU and GPU resources simultaneously performs model retraining and inference serving, with the new data stream (or inference requests) collected from a set of nearby device clients (e.g., surveillance cameras) running the same AI application (e.g., object detection). Since manual labeling for the online data stream is not feasible at the edge, the labels for the retraining are obtained from a ``teacher model'' --- a highly accurate but expensive model (with deeper architecture and larger size). Since the inference latency of the teacher model cannot meet the stringent latency requirements of mission-critical edge AI applications such as safety surveillance, we only use it for labeling. Instead, for the actual inference serving and model retraining, a ``student model'' which is less accurate but more responsive and resource-efficient is adopted. Notably, this philosophy of supervising a small student model with a large teacher model has been widely applied in the community of computer vision.    

To model the periodic behavior of model retraining, we assume that the system works in a time-slotted fashion. Each time slot $t \in \mathcal{T} \triangleq \{1, 2, \cdots, T\}$ denotes a ``retraining window'' that designates model retraining once on the newly collected data. Specifically, as shown in Fig. \ref{fig: Model}, at each time slot $t$, the group of device clients first \WoB{1} upload their inference requests to the edge server. Here we use $D_{(t)}$ to denote the amount of data uploaded at time slot $t$. Without loss of generality, we assume that $D_{(t)} \in [D_{\min}, D_{\max}]$ and $D_{\min}>0$. After receiving the data, the scheduler \WoB{2} determines the configuration for model retraining and inference, based on the data amount ($D_{(t)}$) and the available computational resource at the edge server (denoted as $C_{(t)}$). Afterward, the student model will immediately \WoB{3} return predictions of all inference requests to the clients based on inference configuration from the scheduler. Next, some uniformly random-chosen data according to the retraining configuration will be sent to the teacher model to \WoB{4} get the corresponding high-credit labels (or pseudo-labels). The student model \WoB{5} then updates its weight by retraining the model according to the pseudo-labels and the retraining configuration determined by the scheduler. Then in the next time slot $t+1$, the student model can serve the inference request with retrained model weights, thus improving accuracy. 

\subsection{Resource Allocation Model}\label{subsc:con}

\begin{table}[]
\caption{Notations.}
\label{tab:Notations}
\centering 
\begin{tabular}{l l}
\textbf{Notation} &\textbf{Description}\\\hline
$C_{(t)}$     & the available computational resource in time slot $t$.\\ 
$D_{(t)}$     & the amount of uploaded data at the beginning of $t$.    \\
\multirow{2}{*}{$A_i^T,C_i^T$} & the profit and resource consumption of \\ 
& $i$-th retraining configuration. \\
\multirow{2}{*}{$A_j^I,C_j^I$} & the profit and resource consumption of \\ 
& $j$-th inference configuration.\\\hline  
\multirow{2}{*}{$x_i(t)$} & binary variable indicating whether $i$-th retraining \\ 
& configuration is chosen at time slot $t$.\\
\multirow{2}{*}{$y_j(t)$} &  binary variable indicating whether $j$-th inference\\ 
& configuration is chosen at time slot $t$.\\\hline
\end{tabular}
\vspace{-0.5cm}
\end{table}



When colocating model retraining and inference serving at the edge server, they may compete for the limited computational resource such as CPU and GPU, especially when the data arrival $D_{(t)}$ bursts. Therefore, the resource allocation to model retraining and inference serving faces a fundamental tradeoff between the retrained model’s accuracy and the inference accuracy. Specifically, if we allocate more resource to model retraining to improve its accuracy, the accuracy of the current model inference would diminish due to reduced resource allocation. Vice versa, if we take away resource from model retraining to inference serving, the current inference accuracy would increase but the subsequent inference may decrease due to the reduced accuracy of the retrained model.

The knob to navigate the tradeoff between the retrained model’s accuracy and the inference accuracy is the configuration of both model retraining and inference, which controls the resource-accuracy tradeoff of both model retraining and inference. For model retraining, the configuration refers to the hyperparameters of the training, such as the number of epochs,
training data size \cite{real_khani_21}, or the layers performing back-propagation \cite{etta_niu_22}. For these hyperparameters, a larger value results in higher accuracy, but also at cost of more resource demand. For model inference, the configuration includes the hyperparameters such as frame rate/resolution of the input video, the compressed variant of the model\cite{adaptor_kong_22}, or the early exit point of the branchy model\cite{lien_20}.

\subsubsection{Retraining configuration adaption} In each time slot $t$, the scheduler selects one retraining configuration from the set of feasible configurations $\mathcal{M} \triangleq \{1, 2, \cdots, M\}$. This selection is represented by binary variables $x_i(t) \in \{0, 1\}$, where $x_i(t) = 1$ indicates the $i$-th retraining configuration is selected at time slot $t$. Formally, this can be expressed as:

\begin{equation}
    x_i{(t)}\in\{0,1\}, \quad\forall i\in\mathcal{M}, \ \forall t \in \mathcal{T}, \label{ctr:x}
\end{equation}
\begin{equation}
 \sum_{i=1}^Mx_i{(t)}=1, \quad  \forall t \in \mathcal{T}. \label{ctr:x=1}
\end{equation}

To characterize the resource-accuracy of the $i$-th retraining configuration, we use $C_i^T$ to denote the resource demand (per data sample, measured by FLOPs or MACs) and $A_i^T$ to aid in modeling the tested accuracy.
Given an amount of $D_{(t)}$ data samples at time slot $t$, the total amount of resource demand of model retraining (including pseudo-labeling) is $D_{(t)}C_i^T$. 

\subsubsection{Inference configuration adaption} similar to the retraining configuration, in each time slot $t$, the scheduler selects one inference configuration from the set $\mathcal{N} \triangleq \{1, 2, \cdots, N\}$. This selection is represented by binary variables $y_j(t) \in \{0, 1\}$, where $y_j(t) = 1$ indicates that the $j$-th inference configuration is selected at time slot $t$. Formally, we have: 

\begin{equation}
     y_j{(t)}\in\{0,1\}, \quad\forall j\in\mathcal{N}, \ \forall t \in \mathcal{T},                 \label{ctr:y}
\end{equation}
\begin{equation}
\sum_{j=1}^Ny_j{(t)}=1, \quad  \forall t \in \mathcal{T}. \label{ctr:y=1}
\end{equation} 


The $j$-th inference configuration consumes $D_{(t)}C_j^I$ computational resources (measured by FLOPs or MACs). The profit, $D_{(t)}A_j^I$, is the corresponding $j$-th result of normalizing the model accuracy for all $N$ inference configurations using the maximum value as a reference. Both $D_{(t)}C_j^I$ and $D_{(t)}A_j^I$ are easy to calculate in practice. In our experiments, the MACs and accuracy of different inference configurations on the test dataset, whose distribution is the same as the training dataset, are used to represent $D_{(t)}C_j^I$ and $D_{(t)}A_j^I$.

Let $A_{\min}^T$ and $A_{\max}^T$ denote the minimum and maximum of the set $\{A_i^T \mid \forall i \in \mathcal{M}\}$. Similarly, let $A_{\min}^I$ and $A_{\max}^I$ represent the minimum and maximum of the set $\{A_j^I \mid \forall j \in \mathcal{N}\}$. Then $A_{\min}^T=0$ and $A_{\min}^I>0$ according to a natural assumption that if the computational resources are scarce and for the consideration of satisfaction from users, model retraining is not unnecessary compared with model inference.


\subsubsection{Computational resources constraint} We use $C_{(t)}$ to denote the available computational resource at time slot $t$. And we suppose that there is at least one feasible solution to the following inequality,  regardless of the value of $D_{(t)}$ and $C_{(t)}$:
\begin{equation}
D_{(t)}\sum_{i=1}^MC_i^T x_i{(t)}+ D_{(t)}\sum_{j=1}^NC_j^Iy_j{(t)}\le C_{(t)}, \ \forall t \in \mathcal{T}. \label{ctr:ctdt}
\end{equation}

\subsection{Long-term Accuracy Model} \label{subsc:per}


We model the basic model performance at time slot $t$ as $f\left(\frac{\sum_{\tau=1}^{t-1}D_{(\tau)}\sum_{i=1}^Mx_i(\tau)A_i^T}{\sum_{\tau=1}^{t-1}D_{(\tau)}}\right)$. Our modeling of model performance under various drifts is based on two key observations: (1) Irrespective of the type of drift, model performance declines to a minimum gradually if the model is not retrained regularly. (2) The increase in model performance resulting from training exhibits a diminishing marginal effect\cite{lc_tobias_15}.

To incorporate the first observation into our modeling, it is essential to explore the correlation between current testing data and the previous data used to retrain the model. However, precisely determining the relationship between previous data used for retraining and current testing data is usually challenging or maybe compute-intensive. So similar to the maximum entropy principle, we make the following assumption: \textbf{every previously used retraining configuration has the same effect on current model performance}. That is where $\frac{\sum_{\tau=1}^{t-1}D_{(\tau)}\sum_{i=1}^Mx_i(\tau)A_i^T}{\sum_{\tau=1}^{t-1}D_{(\tau)}}$ comes from. Then to align with the second observation, we introduce a function $f$ that maps the average learning extent of all past data to the current model performance to represent the average influence of the drifts on model performance over time. If there are no drifts, wherein the current test data follows the same distribution as the training dataset on which the model has been fully trained, then model retraining has a small influence on model performance and $f$ reduces to a constant function.




In the study of learning curves\cite{lc_tobias_15}, the expression of the function $f$ can take on various forms, such as power functions like $f(x) = c - ax^{-\alpha}$, exponential functions like $f(x) = \exp(a + \frac{b}{x} + c\log(x))$, logarithmic functions like $f(x) = \log(a\log(x) + b)$, or even a weighted linear combination of these forms. Here, $x$ represents training time, number of iterations, or training dataset size, and $f(x)$ denotes accuracy on the validation set. For a more comprehensive understanding of the potential expressions of the function $f$, please refer to Figure 1 in \cite{lc_tobias_15}. In our study, rather than making assumptions about the exact expression of $f$, we identify that these expressions share a common property: they are concave and increasing. Consequently, we introduce the following assumption about $f$:



\begin{assumption}\label{asp:1}
The function $f(x)$ is increasing, concave, and continuously differentiable over the interval $[0, A^T_{\max}]$, and $f(0)>0$, but its analytical expression is unknown.
\end{assumption}


Other assumptions about the relationship between retraining configuration and model performance may be reasonable as well. For instance, if the current model performance is only related to past data within a time window (e.g., in-context learning), then model performance can be modeled as $f\left(\frac{\sum_{\tau=t-w}^{t-1}D_{(\tau)}\sum_{i=1}^Mx_i(\tau)A_i^T}{\sum_{\tau=1}^{t-1}D_{(\tau)}}\right)$, where $w$ is the window size. And if the current data is more related to nearby data than former data, model performance would be $f\left(\frac{\sum_{\tau=1}^{t-1}D_{(\tau)}\alpha^{t-1-\tau}\sum_{i=1}^Mx_i(\tau)A_i^T}{\sum_{\tau=1}^{t-1}\alpha^{t-1-\tau}D_{(\tau)}}\right)$, where $\alpha$ is a positive decay factor less than 1. A more complex modeling is that the formula of $f$ may change over time, i.e., $f_t\left(\frac{\sum_{\tau=1}^{t-1}D_{(\tau)}\sum_{i=1}^Mx_i(\tau)A_i^T}{\sum_{\tau=1}^{t-1}D_{(\tau)}}\right)$. We leave these variants for future research.


In the presence of an unknown analytical formula for the function $f$, we introduce the following assumptions to facilitate the algorithm design:

\begin{assumption}\label{asp:4}
The value of $f(A^T_{\max})$ is known. And a positive lower bound of $f'(A^T_{\max})$, denoted as $L$, is known.
\end{assumption}

In practice, the accuracy of the trained model on the test dataset with the same distribution as the training dataset serves as an estimate for $f(A^T_{\max})$. This is because this accuracy indicates the model's performance when drifts are absent, effectively representing the highest achievable accuracy ($f(A^T_{\max})$) under the best retraining and inference configurations in the presence of drifts.
The value of $f'(A^T_{\max})$, reflecting the rate of improvement in model accuracy when the model always uses the best retraining configuration under drifts, can only be determined with prior knowledge of the drifts. In our experiments, we set a small constant (e.g., 0.01) as the value of $L$ based on the accuracy improvement on the mentioned test dataset between the last two epochs of the training process. For a typical task, the optimal $L$ can be determined empirically by experimenting with the algorithm with various values of $L$ in the real world. Approximating an unknown value ($L$) is much simpler than approximating an unknown function ($f$).





Moreover, the inference configuration, viewed as the utilization of the model, also plays a significant role in determining its performance. We assume that \textbf{model performance at time slot $t$ is directly proportional to the profit of the inference configuration used at that time}, i.e., $\sum_{j=1}^Ny_j(t)A_j^ID_{(t)}$. While output accuracy with different inference configurations may vary over time, we argue that $A_j^I$ represents the average utilization degree of the $j$-th inference configuration on the model and is assumed to be a constant known in advance. 

Now, the problem of maximizing long-term average accuracy within the constraints of varying computing resources over time, with decision variables ($x_i(t), y_j(t)$) representing the chosen retraining and inference configurations, is formulated as:
\begin{align*}  
\label{pro:P} \tag{P} 
\max_{x_i(t),y_j(t)} & \sum_{t=1}^T f\left(\frac{\sum_{\tau=1}^{t-1}D_{(\tau)}\sum_{i=1}^Mx_i(\tau)A_i^T}{\sum_{\tau=1}^{t-1}D_{(\tau)}}\right)\sum_{j=1}^Ny_j(t)A_j^ID_{(t)}  \\
\mathrm{s.t.} 
&~\text{Constraints } (\ref{ctr:x})-(\ref{ctr:ctdt}) .
\end{align*}

\subsection{Existing Approaches} \label{subsc:appr}

To the best of our knowledge, no online algorithms for a similar problem (\ref{pro:P}) have been proposed in the literature so far. There are three main difficulties when dealing with it: (1) The objective function is nonconvex-nonconcave, as demonstrated in Theorem \ref{thm:non-convex} with its proof in Appendix~\ref{sec:non-convex}. (2) Decision variables are heavily coupled.  (3) The analytical formula for $f$ is commonly unknown in practice.

\begin{theorem}\label{thm:non-convex}
If $f(x)$ is a concave function and defined on $[0, A^T_{\max}]$, then $f(x)y$, defined on $[0, A^T_{\max}]\times [A^I_{\min}, A^I_{\max}]$, is a nonconvex-nonconcave function.
\end{theorem}  
For the third difficulty, a promising technique is \emph{bandit convex optimization}\cite{zeroth_Lattimore_21}. The method usually adds random noise to the decision variable and then estimates the gradient information \cite{bco_zhihua_21} of the function, but these techniques are specifically designed for convex functions and may not be readily applicable to our situation. Moreover, the decision variables of our problem are discrete, and therefore, the technique of adding random noise is not suitable. 
  
The well-known primal-dual\cite{onlinenon_anu_12} method in online algorithms is also not applicable to our problem. Even if we can approximate the analytical expression of $f$, it is difficult to find the dual function of the original function due to the heavy coupling between decision variables. Moreover, since the problem is nonconvex-nonconcave, there may not be strong duality as in linear programming. Then even if the dual function is found and solved, it may violate the original problem constraints. Review\cite{time_simon_20} has more details on the time-varying convex optimization algorithms.









\section{Algorithm Design}\label{sec:ALG}

We follow a three-step procedure to design a lightweight and theoretically guaranteed algorithm: (i) Leverage the concave property of $f$ to move the decision variables $\{x_i(t)\}$ out of $f$ (Lemma \ref{lem:bigger}). Then the problem of interest changes from (\ref{pro:P}) to (\ref{pro:Q}). (ii) Decouple the interaction of $\{x_i(t), y_j(t)\}$ by involving a particular regularization term (Lemma \ref{lem:zt}), with the concerned problem specialized to (\ref{pro:D}). Then \ouralg is proposed to solve (\ref{pro:Dt}), the subproblem of (\ref{pro:D}) in every time slot. (iii) Lemma \ref{lem:PD} and Theorem \ref{thm:cr_OR} are used to facilitate the proof of the competitive ratio of \ouralg.


\subsection{Deal with the Target Function of (\ref{pro:P})}

\begin{lemma}\label{lem:bigger}
$f(x)\le Lx+ g(A^T_{\max})$, where $g(A^T_{\max})=f(A^T_{\max})-LA^T_{\max}$.
\end{lemma}
\begin{proof}
     $f(x)\le f(A^T_{\max})+f'(A^T_{\max})(x-A^T_{\max})$ for the concave property of $f(x)$; Because $x\le A^T_{\max}$ and $f'(A^T_{\max})\ge L>0$, then $f'(A^T_{\max})(x-A^T_{\max})\le L(x-A^T_{\max})$, and thus $f(x)\le Lx+f(A^T_{\max})-LA^T_{\max}\le Lx+g(A^T_{\max})$.
\end{proof}
We use $z_{t-1}$ to denote $\sum_{\tau=1}^{t-1}D_{(\tau)}\sum_{i=1}^Mx_i(\tau)A_i^T$, and then transform the objective function using Lemma~\ref{lem:bigger} for $1 \le t \leq T$. Then we have $f\left(\frac{z_{t-1}}{\sum_{\tau=1}^{t-1}D_{(\tau)}}\right) \le g(A^T_{\max})+\frac{Lz_{t-1}}{\sum_{\tau=1}^{t-1}D_{(\tau)}}$. To further deal with $z_t$, we introduce the positive regularization term $\lambda_t(\sum_{\tau=1}^t D_{(\tau)}\sum_{i=1}^Mx_i(\tau)A_i^T-z_{t})$, where $\lambda_t>0, \forall t$. With this, the new problem (\ref{pro:Q}) is formulated as follows:
\begin{align*}
\label{pro:Q}
\max_{x_i(t),y_j(t),z_{t}}&\  \sum_{t=1}^T g\left(A^T_{\max}\right)\sum_{j=1}^Ny_j(t)A_j^ID_{(t)}\\&+
\sum_{t=2}^TL\frac{z_{t-1}}{\sum_{\tau=1}^{t-1}D_{(\tau)}}\sum_{j=1}^Ny_j(t)A_j^ID_{(t)}\\&+\sum_{t=1}^{T-1}\lambda_t( \sum_{\tau=1}^t D_{(\tau)}\sum_{i=1}^Mx_i(\tau)A_i^T-z_{t}) 
 \tag{Q}
\\
\mathbf{s.t.} \
& \text{Constraints }  (\ref{ctr:x})-(\ref{ctr:ctdt}),  \\ 
& z_{t}\le \sum_{\tau=1}^t D_{(\tau)}A^T_{\max},\quad 1\le t\le T-1. 
\end{align*} 

Using $P^*$ and $Q^*$ to respectively denote the optimal offline objective function value of the problems (\ref{pro:P}) and (\ref{pro:Q}), then the following lemma holds.

\begin{lemma}\label{lem:PD}
Suppose $\lambda_t > 0$ for $1\le t \leq T-1$, then $P^* \leq Q^*$.
\end{lemma}
\begin{proof}
We can prove this by demonstrating that the optimal solution to (\ref{pro:P}) is also a feasible solution for (\ref{pro:Q}). Let $x_i^*(t)$ and $y_j^*(t)$ be the optimal solutions to (\ref{pro:P}), and define $z_t=\sum_{\tau=1}^{t}D_{(\tau)}\sum_{i=1}^Mx_i^*(\tau)A_i^T$. It follows that the optimal solution of (\ref{pro:P}) satisfies the constraints of (\ref{pro:Q}), and the objective function value of (\ref{pro:P}) is less than that of (\ref{pro:Q}) for the optimal solution of (\ref{pro:P}) based on the former inequalities.
\end{proof}

We choose a particular realization of $\lambda_t$ to simplify (\ref{pro:Q}).

\begin{lemma}\label{lem:zt}
when $\lambda_t=L\frac{D_{\min}A_{\min}^I}{tD_{\max}}$, then the corresponding $z_t$ to the optimal solution of (\ref{pro:Q}) is $\sum_{\tau=1}^t D_{(\tau)}A^T_{\max}$ 
\end{lemma}
\begin{proof}
Extracting all terms containing $z_t$ in (\ref{pro:Q}), we have:  $\sum_{t=2}^T L\frac{z_{t-1}}{\sum_{\tau=1}^{t-1}D_{(\tau)}}\sum_{j=1}^Ny_j(t)A_j^ID_{(t)}-\sum_{t=1}^{T-1} \lambda_t  z_{t} =\sum_{t=1}^{T-1} \left[ L\frac{1}{\sum_{\tau=1}^{t}D_{(\tau)}}\sum_{j=1}^Ny_j(t+1)A_j^ID_{(t+1)}- \lambda_t \right]z_{t}$. When $\lambda_t=L\frac{A_{\min}^I D_{\min}}{tD_{\max}}$, the coefficient of $z_t$ is no less than $0$, regardless of the values of $y_j(t+1)$ and $D_{(t+1)}$. To maximize (\ref{pro:Q}), $z_t$ will equal its maximum: $\sum_{\tau=1}^t D_{(\tau)}A^T_{\max}$.
\end{proof}

After setting $\lambda_t=L\frac{D_{\min}A_{\min}^I}{tD_{\max}}$ and $z_t = \sum_{\tau=1}^t D_{(\tau)}A^T_{\max}$ based on Lemma~\ref{lem:zt}, we obtain a specialized version of (\ref{pro:Q}), and we also have $P^*\le D^*$ by Lemma \ref{lem:PD}:
\begin{align*} 
\label{pro:D}
\max_{x_i(t),y_j(t)}  \sum_{t=1}^{T-1}&\lambda_t( \sum_{\tau=1}^t D_{(\tau)}\sum_{i=1}^Mx_i(\tau)A_i^T-\sum_{\tau=1}^t D_{(\tau)}A^T_{\max})\\+\sum_{t=1}^T g (A^T_{\max})&\sum_{j=1}^Ny_j(t)A_j^ID_{(t)}+
\sum_{t=2}^TLA^T_{\max}\sum_{j=1}^Ny_j(t)A_j^ID_{(t)} \tag{D}
\\
\mathbf{s.t.}   
& \quad \text{Constraints }  (\ref{ctr:x})-(\ref{ctr:ctdt})  
\end{align*} 


Then we make some equivalent transformations to the target function of (\ref{pro:D}) and decouple the problem to every time slot:  
\begin{align*}
\label{pro:Dt}
\max_{x_i(t),y_j(t)}&\ V_t\sum_{i=1}^Mx_i(t)A_i^T +W_t\sum_{j=1}^Ny_j(t)A_j^I \tag{Dt}
\\
\mathbf{s.t.} \
& \quad\text{Constraints }(\ref{ctr:x})-(\ref{ctr:ctdt}) \text{ only at $t$ }.
\end{align*}  

\noindent where $V_t = L\frac{D_{\min}A_{\min}^I}{D_{\max}}\left(\sum_{\tau=t}^{T-1} \frac{1}{\tau}\right)$, $W_1 = f\left(A^T_{\max}\right)-LA^T_{\max}$ and $W_t=f\left(A^T_{\max}\right),\forall t>1$.

\subsection{Online Robust Retraining and Inference Co-location}



The problem (\ref{pro:Dt}) can be solved by an exhaustive method with $O(MN)$ complexity. To further speed up the solving, we introduce the following property, as in \cite{coordi_peng_20}:
\begin{property}\label{asp:5}
    $\forall a,b\in \mathcal{M}, C_a^T>C_b^T \Rightarrow A_a^T>A_b^T$; and $\forall a,b\in \mathcal{N}, C_a^I>C_b^I \Rightarrow A_a^I>A_b^I$.
\end{property}

It should be noted that there are some infrequent cases where Property \ref{asp:5} is not satisfied, i.e., better performance can be achieved with fewer resources. For instance, when the image is corrupted by Gaussian noise, downsampling may improve model performance, as illustrated in Table \ref{tab:cifar10}. Similarly, in model training, more iterations may not necessarily lead to better performance\cite{ekya_romil_22}. 
A practical solution\cite{ekya_romil_22} to this issue is regularly measuring the resource-performance profiles of different configurations after model deployment. 

However, since we have assumed the resource requirements and profits of retraining and inference configurations remain constant throughout the whole time span $\mathcal{T}$ in \ref{subsc:con}, configurations that consume more resources yet yield lower profits can be reasonably eliminated before running the algorithm for (\ref{pro:Dt}), ensuring satisfaction of Property \ref{asp:5}. This is because any reasonable algorithm for (\ref{pro:Dt}) would not choose these configurations when there are better alternatives available—ones with equivalent or lower resource requirements but higher profits.

Based on Property \ref{asp:5}, we propose \ouralg (Online Robust Retraining and Inference Co-location), outlined in Algorithm \ref{Alg:ORRIC}. The underlying principle of \ouralg is that the optimal configuration is likely the one about to violate the computational resource constraint. Thus, the optimal configuration can be identified by searching through configurations likely to exceed the computational resource constraint. The proof of the correctness of \ouralg can be found in Appendix \ref{sec:ORRIC}. The complexity of \ouralg is $O(M+N)$: During each iteration of the loop, either $i= i+ 1$ or $j=j-1$. $i$ increases mostly to $M+1$ and $j$ decreases mostly to $0$, the total number of iterations in the loop must be no more than $M+N$.


\begin{algorithm} [t] 
\renewcommand{\algorithmicrequire}{\textbf{Input:}}
\renewcommand{\algorithmicensure}{\textbf{Output:}} 
\begin{algorithmic}[1]
\REQUIRE{$V_t$, $W_t$, $U_{t}=\frac{C_{(t)}}{D_{(t)}}$ and four ascending lists: $\{A_i^T, i\in \mathcal{M}\}$, $\{A_j^I, j\in \mathcal{N}\}$, $\{C_i^T, i\in \mathcal{M}\}$, $\{C_j^I, j\in \mathcal{N}\}$.}
\ENSURE{A pair of retraining and inference configurations.}
\STATE  Initialization: Set $i = 1, j = N, i^* = j^* = K = 0$.
\WHILE{$i\le M$ and  $j\ge 1$} 
\IF{$C_i^T + C_j^I\le U_{t}$}  
    \IF{$V_tA^T_i+W_tA^I_j>K$}
    \STATE $i^*=i$; $j^*=j$; $K=V_t A^T_i+W_t A^I_j$;
    \ENDIF 
    \STATE $i=i + 1;$
\ELSE
    \STATE $j = j-1;$ 
\ENDIF
\ENDWHILE
\STATE  \textbf{return} $i^*, j^*$;
\end{algorithmic}
\caption{\ouralg}
\label{Alg:ORRIC}
\end{algorithm}







In particular, \ouralg aligns with our intuition about the way to allocate limited computing resources to model retraining and inference to optimize long-term model performance. As depicted in Table \ref{tab:heuristic}, \ouralg can be regarded as a combination of four heuristic algorithms, transitioning between them based on the duration of time and the availability of computing resources: 1) Knowledge-Distillation: The teacher model imparts knowledge to the student model without considering resource consumption.
2) Inference-Greedy: Prioritize using a higher configuration for inference and utilize the remaining resources for retraining.
3) Focus-Shift: Shift the focus from retraining to inference as time passes.
4) Inference-Only: This algorithm is actually the traditional computing paradigm that deploys a trained model and then performs inference. 


When the computational resources are sufficient for the use of the best inference and retraining configuration, \ouralg converts to Knowledge-Distillation because $V_t\ge0, W_t>0,\forall t$. When resources are really scarce, e.g., $C_{(t)}=D_{(t)}C_{\min}^I$, \ouralg converts to Inference-Only because $C_{\min}^T=0$ while $C_{\min}^I>0$. When resources are limited (but not scarce) and $T$ is large,  \ouralg converts to Focus-Shift because $\sum_{\tau=t}^{T-1} \frac{1}{\tau}>\ln(T)-\ln(t)$ and $V_T=0$. When resources are limited (but not scarce) and $T$ is small, \ouralg converts to Inference-Greedy because $W_t$ is a constant when $t>1$ while $V_t$ will decrease to 0 with the increasing of $t$.


The translation relationship between \ouralg and the four heuristic algorithms not only illustrates the rationality of \ouralg but also provides insights into the properties of algorithms designed for the model retraining and inference co-location paradigm. We believe that all rational algorithms for this paradigm should similarly translate to these four heuristic algorithms given specific conditions regarding time length and available computing resources, as illustrated in Table \ref{tab:heuristic}.


\begin{table}[]
\caption{ORRIC and Several Heuristic Algorithms.}
\label{tab:heuristic}
\centering
\scalebox{1.00}{
\begin{tabular}{|l|l|l|}\hline
\diagbox{Resources}{T} & Large & Small \\\hline
Sufficient & \multicolumn{2}{c|}{Knowledge-Distillation}\\\hline
Limited & Focus-Shift & Inference-Greedy \\\hline
Scarce & \multicolumn{2}{c|}{Inference-Only}\\\hline
\end{tabular}}
\vspace{-0.5cm}
\end{table}

\noindent \textbf{Remark}:
Our algorithm is an open-loop algorithm that does not leverage feedback from the system. We acknowledge that it is possible to calculate the current accuracy of the student model ($f\left(\frac{\sum_{\tau=1}^{t-1}D_{(\tau)}\sum_{i=1}^Mx_i(\tau)A_i^T}{\sum_{\tau=1}^{t-1}D_{(\tau)}}\right)\sum_{j=1}^Ny_j(t)A_j^ID_{(t)}$)  at every end of time slot $t$ by considering the pseudo-labels output by the teacher model as the ground truth labels, but the relevant mathematical techniques used to incorporate such feedback into the design of algorithms for similar formulas as problem (\ref{pro:P}) are lacking in the existing literature. We leave the research on the closed-loop algorithm to the problem (\ref{pro:P}) as future work.

\section{Performance Analysis}\label{sec:pa}



\begin{definition}
For a maximization problem, the competitive ratio (or CR) $c$ of algorithm $ALG$ is defined as $c \leq \frac{ALG(I)}{OPT(I)}$ for every input $I$, where $OPT$ represents the optimal offline algorithm with complete knowledge of future information.
\end{definition}

\begin{definition}\label{def:tcr}
For a maximization problem, the \textbf{tight competitive ratio} $c$ of algorithm $ALG$ is also a competitive ratio of algorithm $ALG$,  and there is no $c' > c$ such that for every input $I$, $c' \le \frac{ALG(I)}{OPT(I)}$.
\end{definition}
 
\begin{theorem}\label{thm:cr_IO} 
The CR of Inference-Only is $\frac{f(0)}{ f(A^T_{\max} )}$.
\end{theorem}
\begin{proof} Denote $\{x_i^*(t), y_j^*(t)\}$ as the optimal offline solution to (\ref{pro:P}) and $\{x_i(t), y_j(t)\}$ as the solution given by Inference-Only. Then $P^*\le  \sum_{t=1}^T f(A^T_{\max})\sum_{j=1}^Ny_j^*(t)A_j^ID_{(t)} \le  \frac{   f(A^T_{\max} )}{ f(0)}\sum_{t=1}^T f(0)\sum_{j=1}^Ny_j(t)A_j^ID_{(t)} =\frac{   f(A^T_{\max} )}{ f(0)}P$.
\end{proof} 

\begin{theorem}\label{thm:tcr_IO_lower}
An upper bound of the tight competitive ratio of Inference-Only is $\frac{T f(0)}{f(0)+(T-1) f(A^T_{\max})}$.
\end{theorem}
\noindent \textbf{Insight}: The closer $f(0)$ and $f(A^T_{\max})$ are, the closer the competitive ratio ($\frac{f(0)}{f(A^T_{\max})}$) and the upper bound of the tight competitive ratio ($\frac{Tf(0)}{f(0) + (T-1)f(A^T_{\max})}$) of Inference-Only are to 1. This implies that when the drift is very slight, Inference-Only approaches the optimal offline algorithm. The proof of Theorem \ref{thm:tcr_IO_lower} is provided in Appendix \ref{sec:tcr_IO_lower}.

\begin{table*}[ht!]
\centering
\caption{Top-1 Accuracy~(\%) on CIFAR-10 and CIFAR-10-C.}\label{tab:cifar10}  
\scalebox{0.76}{
\tabcolsep3pt
\begin{tabular}{l|   c|c|c |cccc               ccccccccccccccc c}
\hline
Model (Resolution)   & \rotatebox[origin=c]{70}{\parbox{1.4cm}{\centering MACs (M)}}                    & \rotatebox[origin=c]{70}{\parbox{1.4cm}{\centering Latency ($\mu s$)}}  & \rotatebox[origin=c]{70}{original} & \rotatebox[origin=c]{70}{brightness} & \rotatebox[origin=c]{70}{contrast} & \rotatebox[origin=c]{70}{defocus blur} & \rotatebox[origin=c]{70}{elastic transform} & \rotatebox[origin=c]{70}{fog} & \rotatebox[origin=c]{70}{frost} & \rotatebox[origin=c]{70}{gaussian blur} & \rotatebox[origin=c]{70}{gaussian noise} & \rotatebox[origin=c]{70}{glass blur} & \rotatebox[origin=c]{70}{impulse noise} & \rotatebox[origin=c]{70}{jpeg compression} & \rotatebox[origin=c]{70}{motion blur} & \rotatebox[origin=c]{70}{pixelate}  & \rotatebox[origin=c]{70}{saturate} & \rotatebox[origin=c]{70}{shot noise} & \rotatebox[origin=c]{70}{snow} & \rotatebox[origin=c]{70}{spatter} & \rotatebox[origin=c]{70}{speckle noise} & \rotatebox[origin=c]{70}{zoom blur}   & Mean \\ \hline 

MobileNetV2 (20*20)   & 6.35 &  7.54  & 44.93& 42.60& 23.28& 40.47& 39.25& 27.64& 39.46& 38.41& 42.97& 40.33& 41.35& 42.95& 35.84& 43.02& 38.14& 43.29& 39.29& 41.38& 43.08& 41.69 & 39.18\\
MobileNetV2 (24*24)   & 6.71 &  8.37  & 59.38& 54.41& 28.09& 51.26& 49.94& 37.42& 48.49& 48.08& 55.71& 50.18& 53.04& 57.10& 44.07& 55.91& 50.56& 56.77& 50.42& 55.41& 56.78& 49.96 & 50.19 \\
MobileNetV2 (28*28)   & 7.45 & 10.15  & 73.29& 67.94& 38.33& 63.21& 62.48& 49.68& 59.17& 59.23& \textbf{64.53}& 62.21& \textbf{60.38}& 69.31& 53.91& 69.67& 63.68& \textbf{66.57}& 61.70& 67.33& \textbf{66.68}& 61.22 & 61.43 \\
MobileNetV2 (32*32)   & 7.94 & 10.51  & \textbf{79.57}& \textbf{76.00}& \textbf{47.52}& \textbf{71.08}& \textbf{71.91}& \textbf{62.74}& \textbf{62.70}& \textbf{67.02}& 56.28& \textbf{62.90}& 57.38& \textbf{74.71}& \textbf{62.42}& \textbf{76.98}& \textbf{71.61}& 61.98& \textbf{65.83}& \textbf{71.92}& 62.86& \textbf{67.78} & \textbf{65.87} \\

\hline
ResNet50 (20*20)   & 65.76& 17.41 & 54.50& 49.20& 32.26& 50.71& 49.00& 39.31& 44.19& 48.99& 52.23& 49.99& 49.99& 53.04& 45.79& 53.03& 46.06& 52.95& 45.68& 48.92& 52.81& 52.83  &48.26  \\
ResNet50 (24*24)   & 68.96& 19.29 & 71.95& 66.25& 40.68& 62.58& 61.52& 50.54& 60.49& 58.75& 68.26& 62.61& 64.58& 69.64& 54.60& 68.51& 62.09& 69.10& 61.03& 64.42& 68.98& 61.99  &61.93 \\
ResNet50 (28*28)   & 82.01& 24.08 & 79.02& 74.19& 42.74& 66.58& 66.79& 55.34& 66.95& 61.60& 72.89& 68.07& \textbf{66.01}& 75.72& 56.96& 75.12& 69.01& 74.45& 69.11& 70.41& 74.12& 64.91 &66.89\\
ResNet50 (32*32)   & 86.37& 24.09 & \textbf{86.13}& \textbf{83.21}& \textbf{55.34}& \textbf{73.97}& \textbf{76.59}& \textbf{70.41}& \textbf{76.09}& \textbf{68.40}& \textbf{72.94}& \textbf{70.55}& 62.42& \textbf{82.43}& \textbf{66.48}& \textbf{82.33}& \textbf{78.61}& \textbf{76.16}& \textbf{76.44}& \textbf{75.46}& \textbf{75.90}& \textbf{70.13} &\textbf{73.36}\\
\hline  
\ouralg   & - & - & 79.24& 79.06& 52.19& 72.08& 72.35& 67.20& 70.96& 67.51& 68.44& 64.90& 58.99& 75.70& 64.51& 77.23& 73.15& 69.01& 70.46& 71.69& 69.46& 69.69& 69.19\\
\hline  
\end{tabular}} \vspace{-0.3cm}
\end{table*}

\begin{theorem}\label{thm:cr_OR}
The CR of \ouralg is $\frac{(1+\alpha)f(0)}{   f(A^T_{\max} )}$ or $\frac{1}{\frac{   f(A^T_{\max} )}{ f(0)}-\alpha}$, where $\alpha=\frac{LA^T_{\max}D^2_{\min}A_{\min}^I}{f(A^T_{\max})D^2_{\max}A^I_{\max}}$.
\end{theorem}
\begin{proof} 
The basic idea is that we have $P^*\le D^*$ based on Lemma \ref{lem:PD}, and if we prove $D^*\le \frac{1}{c}P$, then $c$ is the competitive ratio. Details see Appendix~\ref{sec:cr_OR}
\end{proof} 

\noindent \textbf{Insight}: First, similarly to Theorem \ref{thm:cr_IO} and \ref{thm:tcr_IO_lower}, if drift is slight, $L$ is close to 0, and then \ouralg reduces to Inference-Only (which is an almost optimal algorithm in this case). Second, \ouralg relies on the precise estimation of the lower bound ($L$) of the degree of drift ($f'(A^T_{\max})$). If the degree of drift is large, then $L$ and $V_t$ are large too if the estimation is precise, making the model pay more attention to retraining to get good future performance. When $f'(A^T_{\max})$ is underestimated too much, the model pays more attention to inference, reducing its future performance due to a lack of retraining, consistent with term $L$ in CR. 
Third,  the term $\frac{D^2_{\min}}{D^2_{\max}}$ suggests that \ouralg performs better with less variability in input data size.


\begin{figure}[t]
    \centerline{\includegraphics[height=1.8in]{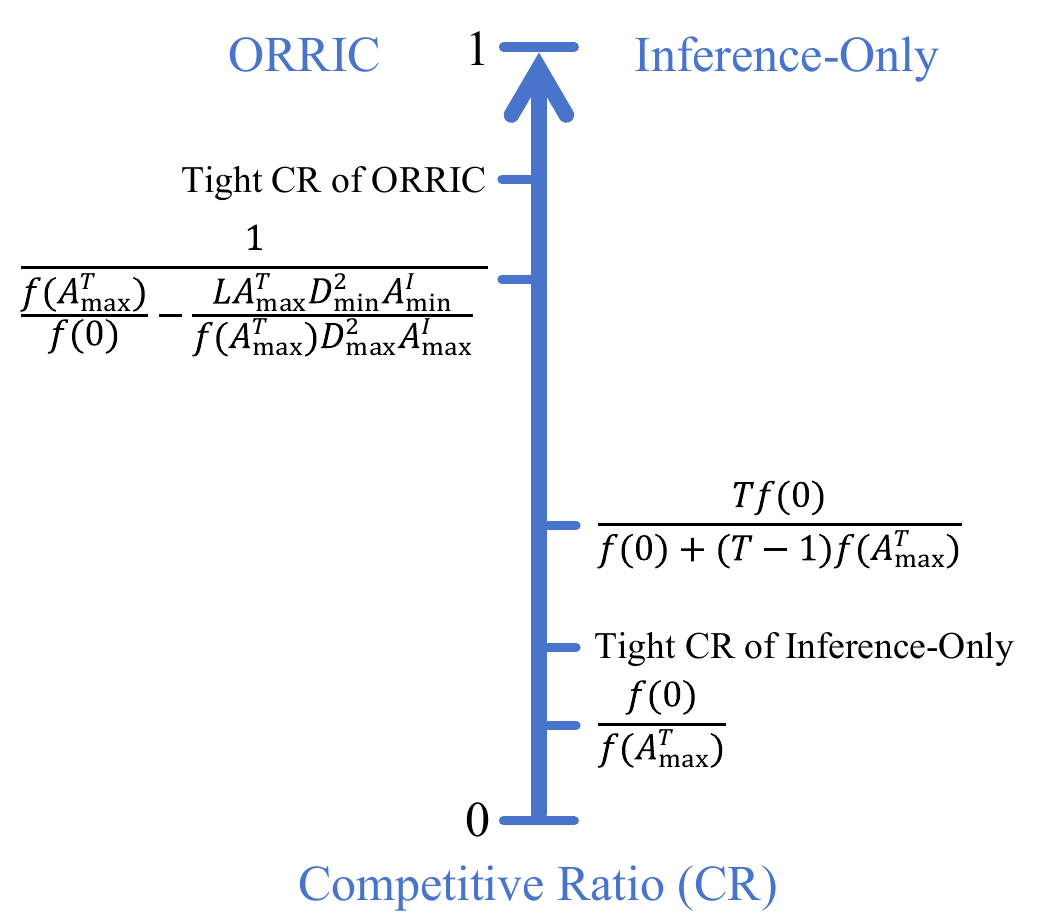}} 
    \caption{Competitive Ratio Result.} 
    \label{fig:cr} 
\vspace{-0.5cm}
\end{figure} 
\begin{corollary}\label{thm:cr_OR_IO}
When $T>\frac{f(A^T_{\max})-f(0)}{\alpha f(0)}$, the tight competitive ratio of \ouralg is strictly better (bigger) than the tight competitive ratio of Inference-Only.

\end{corollary}
\begin{proof}  
Denote $c_1, c_2$ as the tight competitive ratio of \ouralg and Inference-Only. From Definition \ref{def:tcr} and Theorem \ref{thm:cr_OR}, we have $c_1\ge \frac{1}{\frac{   f(A^T_{\max} )}{ f(0)}-\alpha}$. And when $T>\frac{f(A^T_{\max})-f(0)}{\alpha f(0)}$, we have $ \frac{1}{\frac{   f(A^T_{\max} )}{ f(0)}-\alpha}>\frac{T f(0)}{f(0)+(T-1) f(A^T_{\max})}$. Then, based on Theorem \ref{thm:tcr_IO_lower}, we have $\frac{T f(0)}{f(0)+(T-1) f(A^T_{\max})}\ge c_2$. Besides, from Definition \ref{def:tcr} and Theorem \ref{thm:cr_IO}, we have $c_2\ge \frac{f(0)}{ f(A^T_{\max})}$. We summarize our theoretical results in Fig. \ref{fig:cr}.
\end{proof}             

\noindent \textbf{Insight}: Due to the rough approximation to the objective function of (\ref{pro:P}), \ouralg may not fully represent the potentiality of model retraining and inference co-location paradigm. However, the tight competitive ratio of \ouralg still surpasses that of Inference-Only when drift occurs ($L>0$) for a sufficiently lengthy time ($T>\frac{f(A^T_{\max})-f(0)}{\alpha f(0)}$). This implies that, in such scenarios, the worst-case performance of the model retraining and inference co-location paradigm is strictly better than that of the traditional Inference-Only paradigm.

\section{Experiments}\label{sec:exp}
We conduct the experiments to answer the following question: Can model retraining and inference co-location paradigm alleviate the negative effect of data drift on model performance?
 
\subsection{Setup}

CIFAR-10-C\cite{HendrycksD19}, a dataset that is generated by adding 15 common corruptions and 4 extra corruptions to the test dataset of CIFAR-10, is typically used in experiments of out-of-distribution generalization or continual test-time adaptation\cite{Wang_2022_CVPR}.  We treat these corruptions as imitations of data drift.
We first train MobileNetV2 and ResNet50 on the training set of CIFAR-10, then test them on CIFAR-10-C and the test set of CIFAR-10 (the ``original" column in Table \ref{tab:cifar10}) separately. Specially, we use original images (whose resolution is $32*32$) for training, while resized images (whose resolution may be $32*32$, $28*28$, $24*24$, or $20*20$) are used for testing. We do not use more kinds of lower resolution (e.g., $16*16$) because the predictions given by the trained MobileNetV2 and ResNet50 are random in these cases. We also give the computing resource consumption (measured by MACs, which can be calculated by using a third-party library like PyTorch-OpCounter) and latency (measured on our NVIDIA A40 server) results of MobileNetV2 and ResNet50 on a single image at different resolutions. These varying resolutions represent distinct inference configurations ($A^I_j$). The retraining configurations ($A^T_i$) are delineated by the sampling ratios (0, 0.1, 0.2, 0.3, 0.5, 1.0), denoting the portion of uploaded data on the $t$-th time slot ($D_{(t)}$) utilized for one epoch of model retraining.


\subsection{Results}
We compare the following three methods on CIFAR-10-C: Teacher-Only (using ResNet50 for inference and without retraining), Student-Only (using MobileNetV2 for inference and without retraining), and \ouralg (using MobileNetV2 for inference and using ResNet50 to retrain MobileNetV2). 

Without loss of generality, we set $D_{(t)}=1000, \forall t$ for all three methods and $C_{(t)}\sim \mathbb{U}(C_1,C_2), \forall t$ for \ouralg, where $\mathbb{U}(C_1,C_2)$ is a uniform distribution between $C_1$ (the MACs when MobileNetV2 performs inference on 1000 images whose resolution is $32*32$) and $C_2$ (the MACs when ResNet50 performs inference on 1000 images whose resolution is $32*32$). To satisfy Property \ref{asp:5}, we delete the $32*32$ inference configuration and set $f(A^T_{\max})$ to $0.7329$ for the dataset of ``gaussian noise", ``impulse noise", ``shot noise", and ``speckle noise". For other datasets, we set $f(A^T_{\max})$ to $0.7957$ (see the ``original" column). We set $L$ to $0.01$ and set $A_i^T = \beta C_i^T$ (where the $\beta$ is a normalization coefficient, making $\max_i\{A_i^T\}=1$).

The real-time accuracy of these three methods on the ``fog" corruption dataset is shown in Fig. \ref{fig:comp} (a). Because each type of corruption dataset in CIFAR-10-C has 5 severity levels, and the first 10,000 images are at severity 1, while the last 10,000 images are at severity 5\cite{HendrycksD19}, the accuracy of all three methods drops suddenly and periodically. However, the curve of \ouralg is almost always higher than the curve of Student-Only, showing the benefit of model retraining. We also give the average accuracy of \ouralg on other corruption datasets in the last row of Table \ref{tab:cifar10}.

The resource consumption and latency of these three methods on the ``fog" corruption dataset can be calculated using the parameters given in Table \ref{tab:cifar10}, and we report the Accuracy-Cost-Latency trade-off of these three methods while normalizing the maximum value of each axis to 1, see Fig. \ref{fig:comp} (b). \ouralg surpasses or equals the Student-Only algorithm in terms of accuracy and latency while utilizing idle available computing resources. \ouralg surpasses the Teacher-Only algorithm in terms of cost and latency while maintaining good accuracy. In general, the model retraining and inference co-location paradigm can utilize idle available resources to improve model accuracy while maintaining low latency, thereby alleviating the negative impact of drift on accuracy.


 

\begin{figure}[t]
  \centering
\subfigure [Real-time Accuracy Results Comparison.]{\includegraphics[width=0.62\linewidth]{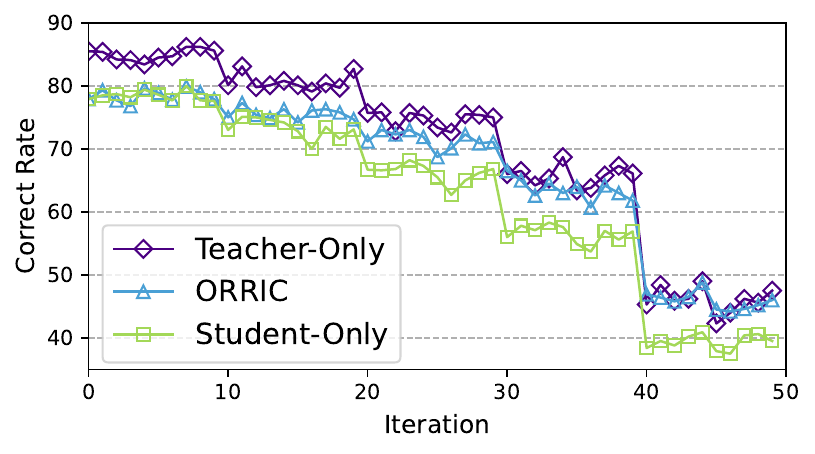}}\ 
   \subfigure[Accuracy-Cost-Latency Trade-off Comparison.]{\includegraphics[width=0.36\linewidth]{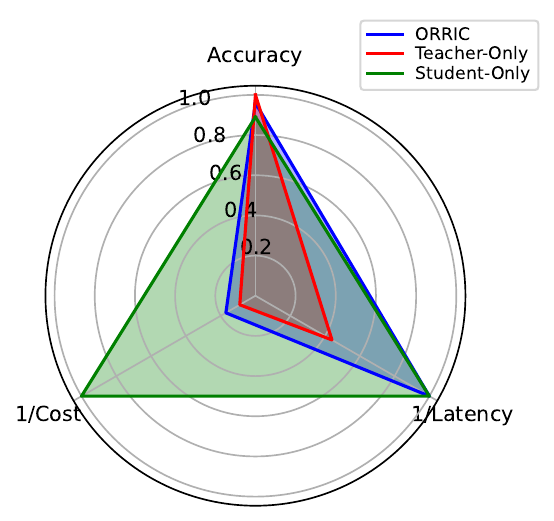}}\\
      \caption{Results on the ``fog" Corruption Dataset of CIFAR-10-C.}
    \label{fig:comp}
  \vspace{-15pt}
\end{figure}



\section{Conclusion}
In this paper, we study the online allocation in the model retraining and inference co-location paradigm. We model the current model performance as a function of past retraining configuration and current inference configuration and then propose a linear complexity online algorithm (named \ouralg). Our competitive analysis implies the advantages and applications of model retraining and inference co-location paradigm over the traditional Inference-Only paradigm. Experiments on the CIFAR-10-C validate the effectiveness of model retraining and inference co-location in drift scenarios.

\section{Acknowledgment}

The authors appreciate the reviewers for their insightful and valuable comments. Discussions with Kongyange Zhao, Tao Ouyang and Qing Ling are gratefully acknowledged.

\newpage 
\appendix
\subsection{Proof of Theorem \ref{thm:non-convex}}\label{sec:non-convex}
\begin{proof}
Supposing $(x_1,y_1)$ and $(x_2,y_2)$ are two points in the domain of $f(x)y$, and denoting $\bar x = \alpha x_1 + (1-\alpha)x_2$ and $\bar y = \alpha y_1 + (1-\alpha)y_2$, where $0<\alpha<1$, then $E=
     f(\bar x)\bar y - \left[\alpha f(x_1)y_1+ (1-\alpha)f(x_2)y_2\right]  
    = \alpha \left[f(\bar x) - f(x_1)\right] y_1 + (1-\alpha)\left[f(\bar x) - f(x_2)\right]y_2$. 
If there exist values for $\alpha, x_1, x_2,y_1, y_2$ that make $E$ less than $0$, then $f(x)y$ is nonconcave. We can also prove that $f(x)y$ is nonconvex in the same way. Therefore, we can conclude that $f(x)y$ is nonconvex-nonconcave. 

If $f(x)$ is twice-differentiable, the indefiniteness of the Hessian matrix of $f(x)y$ can prove the theorem too.
\end{proof}


\subsection{Proof of Theorem \ref{thm:tcr_IO_lower}}\label{sec:tcr_IO_lower}

\begin{proof}
We show this fact using proof by contradiction. Suppose a situation where computing resources are sufficient and $D_{(t)}=D$ on every time slot, then  $P^* = f(0)A_{\max}^ID_{(1)}+\sum_{t=2}^T f(A^T_{\max})A_{\max}^ID_{(t)}=(f(0)+\sum_{t=2}^T f(A^T_{\max}))A_{\max}^ID$, while $P = \sum_{t=1}^T f(0)A_{\max}^ID_{(t)}=\left(\sum_{t=1}^T f(0)\right)A_{\max}^ID$, then we have: $\frac{P}{P^*} = \frac{\sum_{t=1}^T f(0)}{f(0)+\sum_{t=2}^T f(A^T_{\max})}=\frac{T f(0)}{f(0)+(T-1) f(A^T_{\max})}$. Suppose the tight competitive ratio of the Inference-Only algorithm (denoted as $\bar c$) is strictly bigger than $\frac{T f(0)}{f(0)+(T-1) f(A^T_{\max})}$, which means for any input, $P\ge \bar c P^*>\frac{T f(0)}{f(0)+(T-1) f(A^T_{\max})}P^*$, but we have found an input that makes $P=\frac{T f(0)}{f(0)+(T-1) f(A^T_{\max})}P^*$, the contradiction has arisen. Then we show an upper bound of the tight competitive ratio  $\bar c$ of the Inference-Only algorithm is  $\frac{T f(0)}{f(0)+(T-1) f(A^T_{\max})}$.  
\end{proof} 

\subsection{Proof of the Correctness of \ouralg}\label{sec:ORRIC}
\begin{proof}  


Let's assume that the optimal configuration indices for retraining and inference in $Dt$ are $a$ and $b$, so $C_{a}^T + C_{b}^I \le U$. We only need to demonstrate that \ouralg must have explored the pair $(a, b, VA^T_a + WA^I_b)$. \ouralg terminates when either $i>M$ or $j<1$. Let's consider the scenario where it terminates with $j<1$ (the case for terminating with $i>M$ is similar). In this case, $j$ will decrease from $N$ to $0$. When $j$ reaches $b$, let's assume that $i=a_1$ at this moment.

First case: If $a_1 \le a$, then $C_{a_1}^T + C_{b}^I \le U$. According to the algorithm, $i$ will start increasing from $a_1$ until $C_{i}^T + C_{b}^I > U$ or until $i>M$, whichever happens first. At this point, $i>a$, so $(a, b, VA^T_a + WA^I_b)$ must have been explored by \ouralg.

Second case: If $a_1 > a$, then the previous iteration is $(a_1, b+1)$ (where $C_{a_1}^T + C_{b+1}^I > C_{a}^T + C_{b+1}^I > U$). And the former iteration of it won't be $(a_1-1, b+1)$ (since $C_{a_1-1}^T + C_{b+1}^I \ge C_{a}^T + C_{b+1}^I > U$, so $(a_1-1, b)$ is next to $(a_1-1, b+1)$). Therefore, the next pairs are $(a_1, b+2)$, $(a_1, b+3)$, and so on until $(a_1, N)$ is reached. At this point, the pair before must be $(a_1-1, N)$, and $(a_1-2, N)$, and so on until $a_1-k$ is found such that $C_{a_1-k}^T + C_{N}^I < U$. In this case, $(a, N)$ must be present in these iterations. However, according to the algorithm, the next iteration from $(a, N)$ is $(a, N-1)$, not $(a+1, N)$. Therefore, this case is not possible. 
\end{proof}

\subsection{Proof of Theorem \ref{thm:cr_OR}}\label{sec:cr_OR}
\begin{proof}

\begin{align*}
D^*&= \underbracket{
f\left(A^T_{\max}\right)\sum_{j=1}^Ny_j(1)A_j^ID_{(1)}}_{A_1} \underbracket{-LA^T_{\max}\sum_{j=1}^Ny_j(1)A_j^ID_{(1)}}_{B_1} 
\\
& +
\underbracket{\sum_{t=2}^T f\left(A^T_{\max}\right)\sum_{j=1}^Ny_j(t)A_j^ID_{(t)}}_{A_2}
\\
&+\underbracket{\sum_{t=1}^{T-1}L \frac{D_{\min}A_{\min}^I}{D_{\max}}\frac{1}{t}\left( \sum_{\tau=1}^t D_{(\tau)}\sum_{i=1}^Mx_i(\tau)A_i^T\right)}_{C}
\\ 
&\underbracket{-\sum_{t=1}^{T-1}L\frac{D_{\min}A_{\min}^I}{D_{\max}}\frac{1}{t}\sum_{\tau=1}^t D_{(\tau)}A^T_{\max}}_{B_2} 
\end{align*}


For term $B_1$, we have $B_1\le -\frac{LA^T_{\max}D^2_{\min}A_{\min}^I}{D_{\max}}$ by $\sum_{j=1}^Ny_j(1)A_j^ID_{(1)}\ge D_{\min}A_{\min}^I\ge D_{\min}A_{\min}^I\frac{D_{\min}}{D_{\max}}$. For term $B_2$,  we have $B_2\le-(T-1)\frac{LA^T_{\max}D^2_{\min}A_{\min}^I}{D_{\max}}$ due to $\sum_{\tau=1}^t D_{(\tau)}\ge tD_{\min}$. Finally $B_1+B_2\le-T\frac{LA^T_{\max}D^2_{\min}A_{\min}^I}{D_{\max}}$.

For term $C$, since the increasing and concave property of $f$ (Assumption \ref{asp:1}), $\frac{\sum_{\tau=1}^t D_{(\tau)}\sum_{i=1}^M x_i(\tau)A_i^T}{\sum_{\tau=1}^t D_{(\tau)}}\le A^T_{\max}$ by the definition of $A^T_{\max}$ and Assumption \ref{asp:4}, we have the following fact: $L\le f'\left({A^T_{\max}}\right)\le f'\left(\frac{\sum_{\tau=1}^t D_{(\tau)}\sum_{i=1}^M x_i(\tau)A_i^T}{\sum_{\tau=1}^t D_{(\tau)}}\right)$. Combining this fact and $\frac{D_{\min}A_{\min}^I}{D_{\max}}\frac{1}{t}<\frac{\sum_{j=1}^N y_j(t+1)A_j^ID_{(t+1)}}{\sum_{\tau=1}^t D_{(\tau)}}$, we get: 
$C\le \sum_{t=1}^{T-1}\left(f(0)+ f'\left(\frac{\sum_{\tau=1}^t D_{(\tau)}\sum_{i=1}^Mx_i(\tau)A_i^T}{\sum_{\tau=1}^t D_{(\tau)}}\right)\right.
\\\left.\frac{\sum_{\tau=1}^t D_{(\tau)}\sum_{i=1}^M x_i(\tau)A_i^T}{\sum_{\tau=1}^t D_{(\tau)}}\vphantom{\frac{\sum_{\tau=1}^t D_{(\tau)}\sum_{i=1}^M x_i(\tau)A_i^T}{\sum_{\tau=1}^t D_{(\tau)}}}\right)\sum_{j=1}^Ny_j(t+1)A_j^ID_{(t+1)}
-\sum_{t=1}^{T-1}f(0)\sum_{j=1}^Ny_j(t+1)A_j^ID_{(t+1)}$. 


Since the assumed concave property of $f$ (Assumption \ref{asp:1}), we have $f(0)\le f(x)+(0-x)f'(x)$ , i.e., $f(0)+xf'(x)\le f(x)$. Then, $C\le  \sum_{t= 1}^{T-1}
 f\left(\frac{\sum_{\tau=1}^t D_{(\tau)}\sum_{i=1}^Mx_i(\tau)A_i^T}{\sum_{\tau=1}^t D_{(\tau)}}\right) \sum_{j=1}^N y_j(t+1)
A_j^ID_{(t+1)}-\sum_{t=1}^{T-1}f(0)\sum_{j=1}^N y_j(t+1)A_j^ID_{(t+1)}   
= P-\sum_{t=0}^{T-1}f(0)\sum_{j=1}^Ny_j(t+1)A_j^ID_{(t+1)}$.

 
Based on all of the above analysis, we have: $P^*\le D^*=B_1+B_2+C+A_1+A_2
\le -T\frac{LA^T_{\max}D^2_{\min}A_{\min}^I}{D_{\max}} + P- \sum_{t=0}^{T-1}f(0)\sum_{j=1}^Ny_j(t+1)A_j^ID_{(t+1)}  +\sum_{t=1}^T f(A^T_{\max})\sum_{j=1}^N y_j(t) A_j^ID_{(t)}
= -T\frac{LA^T_{\max}D^2_{\min}A_{\min}^I}{D_{\max}} +P+\frac{   f(A^T_{\max} )-f(0)}{ f(0)}\sum_{t=1}^T f(0)\sum_{j=1}^Ny_j(t)A_j^ID_{(t)}
\le -T\frac{LA^T_{\max}D^2_{\min}A_{\min}^I}{D_{\max}} + P+\frac{   f(A^T_{\max} )-f(0)}{ f(0)}P
=\frac{   f(A^T_{\max} )}{ f(0)}P
-\alpha Tf(A^T_{\max})A^I_{\max}D_{\max}$, where $\alpha=\frac{LA^T_{\max}D^2_{\min}A_{\min}^I}{f(A^T_{\max})D^2_{\max}A^I_{\max}}$.
Further from the fact that $P\le P^*\le Tf(A^T_{\max})A^I_{\max}D_{\max}$, we have: 1) $P^* \le\frac{f(A^T_{\max} )}{ f(0)}P-\alpha P^*$, 2) $P^* \le
\frac{f(A^T_{\max} )}{ f(0)}P-\alpha P$.
Then we prove that the competitive ratio of \ouralg is $\frac{(1+\alpha)f(0)}{   f(A^T_{\max} )}$ or $\frac{1}{\frac{   f(A^T_{\max} )}{ f(0)}-\alpha}$, where $\alpha=\frac{LA^T_{\max}D^2_{\min}A_{\min}^I}{f(A^T_{\max})D^2_{\max}A^I_{\max}}$.
\end{proof}

\clearpage
\bibliographystyle{IEEEtran}
\bibliography{main}
\end{document}